\documentclass[conference]{IEEEtran}
\IEEEoverridecommandlockouts
\usepackage{caption}
\usepackage{subcaption}
\ifCLASSINFOpdf
\else
   \usepackage[dvips]{graphicx}
\fi
\usepackage{url}
\usepackage{amsmath}
\usepackage{amsthm}
\usepackage{amssymb}
\usepackage[ruled,vlined]{algorithm2e}
\newtheorem{theorem}{Theorem}
\newtheorem{lemma}{Lemma}
\newtheorem{assumption}{Assumption}
\def\BibTeX{{\rm B\kern-.05em{\sc i\kern-.025em b}\kern-.08em T\kern-.1667em\lower.7ex\hbox{E}\kern-.125emX}}

\usepackage{graphicx,color}

\DeclareMathOperator*{\argmin}{arg\,min}
\usepackage{textcomp}
\newcommand{\distas}[1]{\mathbin{\overset{#1}{\kern\z@\sim}}}%
\newsavebox{\mybox}\newsavebox{\mysim}
\newcommand{\distras}[1]{%
  \savebox{\mybox}{\hbox{\kern3pt$\scriptstyle#1$\kern3pt}}%
  \savebox{\mysim}{\hbox{$\sim$}}%
  \mathbin{\overset{#1}{\kern\z@\resizebox{\wd\mybox}{\ht\mysim}{$\sim$}}}%
}
 \def\BibTeX{{\rm B\kern-.05em{\sc i\kern-.025em b}\kern-.08em
     T\kern-.1667em\lower.7ex\hbox{E}\kern-.125emX}}

\usepackage{cite}

\title{Robust Networked Federated Learning for Localization}
\author{Reza Mirzaeifard$^{\star}$, Naveen K. D. Venkategowda$^\S$, Stefan Werner$^{\star}$ \thanks{This work was supported by the Research Council of Norway.}\\

$^{\star}$Dept. of Electronic Systems, Norwegian University of Science and Technology-NTNU, Norway \\  $^\S$Department of Science and Technology,  Linköping University, Sweden \\
E-mails: \{reza.mirzaeifard,  stefan.werner\}@ntnu.no, naveen.venkategowda@liu.se
}

\begin{document}

\maketitle
%\IEEEpubidadjcol
 \begin{abstract}
 This paper addresses the problem of localization, which is inherently non-convex and non-smooth in a federated setting where the data is distributed across a multitude of devices. Due to the decentralized nature of federated environments, distributed learning becomes essential for scalability and adaptability. Moreover, these environments are often plagued by outlier data, which presents substantial challenges to conventional methods, particularly in maintaining estimation accuracy and ensuring algorithm convergence. To mitigate these challenges, we propose a method that adopts an \(L_1\)-norm robust formulation within a distributed sub-gradient framework, explicitly designed to handle these obstacles. Our approach addresses the problem in its original form, without resorting to iterative simplifications or approximations, resulting in enhanced computational efficiency and improved estimation accuracy. We demonstrate that our method converges to a stationary point, highlighting its effectiveness and reliability. Through numerical simulations, we confirm the superior performance of our approach, notably in outlier-rich environments, which surpasses existing state-of-the-art localization methods. 
 \end{abstract}

%\IEEEpeerreviewmaketitle
\section{Introduction}
The increasing prevalence of edge devices in the era of the internet-of-things (IoT) and cyber-physical systems has resulted in a significant upsurge in available data. These devices, functioning within a federated network, present unique challenges.  In particular, traditional machine learning methodologies often require a central server for data processing, raising privacy concerns \cite{yin2021comprehensive}.  In response to these concerns, federated learning (FL) has emerged as an innovative solution \cite{zhou2023decentralized,zhao2022participant,li2020federated} for edge devices to collectively train a global model using their locally stored data, eliminating the need for direct data sharing \cite{zhou2023decentralized,li2020federated}. However, federated settings are frequently besieged by outlier data, often exhibiting heavy-tailed distributions \cite{zhao2022participant,tsouvalas2022federated}. These outliers can considerably distort the accuracy of learning outcomes. Nevertheless, developing robust and computationally efficient solutions for managing outlier data in FL remains a critical research question, which this paper aims to address in the context of localization.

%\textbf{\textcolor{red} {Add: What is localization problem?}}

Localization is the process of estimating the location of an object or an event within a given environment. This problem has given rise to a broad range of applications in vehicle localization \cite{kaiwartya2018geometry}, aviation \cite{strohmeier2018k}, healthcare \cite{han2012landmark}, environmental \cite{zhou2010scalable}, and industrial fields \cite{petersen2007requirements}. %Localization is especially important for IoT, where sensor nodes can collect meaningful data when paired with accurate geolocation \cite{li2020toward}. 
Based on the type of measurements used to estimate position, the localization problem has different formulations which include time-of-arrival (ToA) \cite{pun2021local}, received-signal-strength (RSS) \cite{yin2017received}, time-difference-of-arrival (TDoA) \cite{okello2011comparison}, angle-of-arrival (AoA) \cite{xu2017optimal}, and frequency-difference-of-arrival (FDoA) \cite{lin2022underwater}.  %Each method presents unique benefits and challenges that influence the choice of technique to fit the needs of a given application.  
In IoT applications, ToA and RSS are particularly effective as they provide relatively simple and cost-effective ways to estimate distances using the existing communication infrastructure. Therefore, in this paper, we focus on distance-based localization that inherently encompasses ToA and RSS techniques, as they can be interpreted as distance measurements \cite{luke2017simple}. This approach finds relevance across a spectrum of practical scenarios and continues to be an area of research interest. 

Conventionally, various methods have been proposed for localization problems, including the parallel projection method \cite{jia2010set}, projection-onto-convex-sets method \cite{blatt2006energy}, nearest local minimum \cite{shi2008distributed}, boundary-of-convex-sets \cite{wang2009sensor}, recursive weighted least-squares algorithms \cite{wang2014decentralized}, augmented Lagrangian-based method for localization \cite{luke2017simple}, and iterative re-weighted least-squares \cite{zaeemzadeh2017robust}. However, these methods have limitations, such as being centralized or sequentially updating estimates, which hinders parallelization and flexibility. More recently, a parallel distributed alternating projection algorithm (DAPA) has been proposed for the distributed case, formulating the localization problem as a ring intersection problem \cite{zhang2015distributed}. Despite the advancements of DAPA, challenges remain, particularly in handling three-dimensional localization problems and targets lying outside the sensor's convex hull \cite{zhang2019sensor}. To address these limitations, an alternative approach called EL-ADMM has emerged. EL-ADMM directly solves the non-convex and non-smooth event localization problem using the Alternating Direction Method of Multipliers \cite{zhang2019sensor}. By bypassing the need for high-computation convex relaxation techniques, EL-ADMM shows promise in handling three-dimensional environments \cite{zhang2019sensor}. However, EL-ADMM still faces challenges in outlier handling and lacks comprehensive convergence proof. Given these constraints, there is a compelling research imperative to develop a robust, distributed algorithm capable of efficiently handling three-dimensional localization problems with outliers.

In this paper, we address a challenging class of robust localization problems that require tackling non-convex and non-smooth optimization within a federated setting. The decentralization, energy constraints, and the presence of outliers further exacerbate the difficulty of accurately estimating locations in such settings. To overcome these challenges, we propose a novel distributed sub-gradient-based algorithm specifically designed to address this problem. Our contributions can be summarized as follows:
\begin{itemize}
   \item  \textbf{Efficient Optimization Algorithm}: We introduce a novel optimization algorithm that operates within a single loop framework. This algorithm directly solves the localization problem in its original form, utilizing simple updating steps. The use of this approach leads to increased accuracy and eliminates the need for iterative approximation processes.
    \item \textbf{Theoretical Analysis}: We provide comprehensive theoretical insights into the convergence behavior of our algorithm. Through rigorous mathematical analysis, we elucidate the conditions under which our method is guaranteed to converge, thereby offering deeper insights into its reliability and robustness. 
    \item \textbf{Empirical Validation}: We substantiate our theoretical results through extensive numerical simulations. Notably, our algorithm exhibits exceptional resilience in the presence of heavy-tailed noise and outliers. Furthermore, a comparative study with the state-of-the-art EL-ADMM algorithm highlights the superior accuracy and comparable fast convergence rate of our approach.
\end{itemize}
%These three contributions highlight the significant advantages of our algorithm and provide strong motivation for its application in various robust localization scenarios.

\noindent\textit{\textbf{Mathematical Notations}}: Scalars are denoted by lowercase letters, column vectors by bold lowercase letters, and matrices by bold uppercase letters. The transpose of a matrix is signified by $(\cdot)^\text{T}$. The $j$th column of a matrix $\mathbf{A}$ is represented as $\mathbf{a}{j}$. The element in the $i$th row and $j$th column of $\mathbf{A}$ is represented as $a_{ij}$. The sub-gradient of a function $f(\cdot)$ at a given point $u$ is signified by $\partial f(u)$.

\section{Preliminaries}
We consider the problem of localization with $L$ sensors, where the location of the $i$-th sensor is denoted by $\mathbf{a}_i$, for  $i = 1, 2,\dots, L$, and $A = \{\mathbf{a}_1, \mathbf{a}_2, \dots, \mathbf{a}_L\}$ represents the complete set of sensors’ locations. The unknown source is located at $\mathbf{x} \in \mathbb{R}^n$, which must be estimated using the range information between the sensors and the source. Let  $d_i > 0$, $i = 1, 2, \dots, L$, represents a noisy observation representing the range between source location $\mathbf{x}$ and $i$-th sensor location $\mathbf{a}_i$. The range $d_i$ can be expressed as
\begin{equation}\label{eq1}
    d_i = \|\mathbf{x}-\mathbf{a}_i\|+\epsilon_i, \quad   i = 1,2,\dots,L,
\end{equation}
where $\epsilon_i$ denotes the noise associated with each range observation.

Additionally, the network of $L$ sensors is modeled as an unweighted undirected graph $\mathcal{G}$ consisting of vertices $\mathcal{V} = {1, \cdots, L}$ and bidirectional communication links represented by the edge set $\mathcal{E}$. Each sensor $i \in \mathcal{V}$ can communicate with those in its neighborhood $\mathcal{N}_i$ of cardinality $|\mathcal{N}_i|$. We assume that graph $\mathcal{G}$ is strongly connected.

The maximum-likelihood estimate for the source location $\mathbf{x}$  with additive independent and identically distributed Gaussian noise affecting range measurements can be determined as the solution of the optimization problem \cite{luke2017simple}:
\begin{equation}\label{eq2}
\hat{\mathbf{x}}=\argmin_{\mathbf{x}} \frac{1}{L} \sum_{i=1}^{L} \left( \|\mathbf{x}-\mathbf{a}_i\|-d_i\right)^2 .
\end{equation}
However, this approach encounters challenges when faced with outlier measurements that do not conform to a Gaussian distribution. These outliers can drastically bias location estimate due to the magnified impact of their squares in each term. One could use conventional filtering techniques to reduce this effect. However, such filtering methods often presuppose knowledge about the noise characteristics, which is not always available or accurate. Additionally, filtering can induce a 'masking effect', in which crucial signal components, particularly those smaller ones sharing frequency characteristics with the noise, are suppressed or obscured \cite{zaeemzadeh2017robust}. Such an outcome is undesirable in our context as it may lead to the loss of vital information \cite{zaeemzadeh2017robust}. To counteract this, robust estimation theory offers a way of softly rejecting outliers by using the $L_1$ loss \cite{oguz2011robust} as given by
\begin{equation}\label{eq3}
\hat{\mathbf{x}}=\argmin_{\mathbf{x}}  \frac{1}{L} \sum_{i=1}^{L} |\|\mathbf{x}-\mathbf{a}_i\|-d_i| =\argmin_{\mathbf{x}}   \sum_{i=1}^{L} f_i(\mathbf{x}),
\end{equation}
where the local objective function $f_i(\mathbf{x})$ at sensor $i$ is defined as $f_i(\mathbf{x})= \frac{1}{L}|\|\mathbf{x}-\mathbf{a}_i\|-d_i|$.
The $L_1$ objective function diminishes the impact of outliers by minimizing the average absolute discrepancy between the Euclidean distance from $\mathbf{x}$ to each sensor $\mathbf{a}_i$ and the corresponding observed range $d_i$. This approach is particularly beneficial as it ensures that smaller yet crucial variations in the data are preserved, and not overshadowed by outliers.

 Even though the $L_1$ robust formulation can enhance accuracy and handle outliers more efficiently, it is inherently non-smooth and non-convex and not even a difference of convex functions, which makes \eqref{eq3} difficult to solve. Traditionally, convex relaxation approaches have been considered to tackle this problem \cite{oguz2011robust,zaeemzadeh2017robust}. These methods, however, can compromise accuracy and are often time-consuming and computation-intensive iterative processes. In the distributed case, the problem exacerbates as one must ensure consensus among the nodes as well as convergence to a solution. To that end, this paper proposes a method for addressing the problem directly with the $L_1$ loss in a distributed manner.

%Considering that each $a_i$ and $d_i$ are exclusively available at node $i$, we define the local function $f_i(\mathbf{x})= \frac{1}{L}|\|\mathbf{x}-\mathbf{a}_i\|-d_i|$. Accordingly, the optimization problem shown in \eqref{eq3} can be reformulated as the summation of local objective functions:

%\begin{equation}\label{eq4}
%\hat{\mathbf{x}}=\argmin_{\mathbf{x}}   %\sum_{i=1}^{L} f_i(\mathbf{x}).
%\end{equation}

\section{Distributed Robust Localization}
The proposed algorithm for localization in a distributed setting called distributed sub-gradient method for robust localization (DSRL), updates each node's estimate of the actual parameter $\mathbf{x}^{*}$ through a diffusion step and a sub-gradient update step. In the diffusion step, each node uses a diffusion variable, $\mathbf{v}_i$, to incorporate information from its neighboring nodes. At each time instant $k$, the diffusion step takes place as follows:
\begin{equation}\label{eq5}
\mathbf{v}_i^{(k)}=\mathbf{x}_i^{(k)} + \alpha^{(k)}
\sum_{j \in \mathcal{N}_i}\left(\mathbf{x}_j^{(k)}-\mathbf{x}_i^{(k)}\right).
\end{equation}
In this formula, $\mathbf{x}_i^{(k)}$ is the current estimate at node $i$, while $\mathbf{x}_j^{(k)}$ are the estimates at each neighboring node $j$. The term $\alpha^{(k)}$ is a scalar weight parameter that adjusts how strongly the estimates of the neighbors influence the node's update. The diffusion step seeks to decrease the discrepancy between each node's estimate and the average estimate of its neighbors, working as a type of approximation to the average consensus algorithm. Instead of a simple average, we apply a weighted adjustment based on the difference between a node's estimate and its neighbors' estimates, thereby aiding the nodes to reach a consensus on the parameter estimate. 

In the subsequent step, termed the sub-gradient step, each node updates its estimate of the true parameter $\mathbf{x}^{*}$ in the manner depicted below:
\begin{equation}\label{eq6}
\mathbf{x}_i^{(k+1)}=\mathbf{v}_i^{(k)}  - \beta^{(k)}
\mathbf{g}_i^{(k)} ,
\end{equation}
where $\beta^{(k)}$ is the step-size at iteration $k$ and $\mathbf{g}_i^{(k)} \in \partial f_i(\mathbf{x}_i^{(k)})$ is any element of the sub-differential set of $f_i(\cdot)$. Each function $f_i(\cdot)$ is neither convex nor smooth, so the distributed sub-gradient method proposed in \cite{swenson2022distributed} is adopted for iterative computation of the estimate of $\mathbf{x}$ at each node. The sub-gradient of the local robust localization function $f_i(\cdot)$ with respect to the coefficient $\mathbf{x}_i^{(k)}$ is computed as follows:
\begin{equation}\label{eq7}
 \partial f_i(\mathbf{x})=
\begin{cases}
\frac{\mathbf{x}-\mathbf{a}_i}{L\|\mathbf{x}-\mathbf{a}_i\|}\text{sign}(\|\mathbf{x}-\mathbf{a}_i\|-d_i), & \|\mathbf{x}-\mathbf{a}_i\|>0
\\
\mathbf{0}. &  \|\mathbf{x}-\mathbf{a}_i\|=0
\end{cases}
\end{equation}
 Algorithm \ref{alg:1} summarizes the proposed method for solving distributed robust localization.
 \begin{algorithm}[t]
 \caption{Distributed Sub-gradient Method for Robust Localization (DSRL)}
 \label{alg:1}
\SetAlgoLined
Initialize $\mathbf{x}_i^{(1)}$ for each node $i$, the weight parameter $\{\alpha_k\}_{k=1}^{K}$, the step-size $\{\beta_k\}_{k=1}^{K}$ and the number of iterations $K$\;
 \For{$k=1,\cdots,K$}{
%  Each agent $l \in [1,\ldots,L]$ update its parameters as following:
\For{$i=1,\cdots,L$}{
  Receive $\mathbf{x}_j$ from neighbors in $\mathcal{N}_i$\;
   Update $\mathbf{v}_i^{(k)}$ by \eqref{eq5}\;
   Update $\mathbf{g}_i^{(k)}$ by  \eqref{eq7}\;
   Update $\mathbf{x}_i^{(k+1)}$ by \eqref{eq6}\;
 }
 }
\end{algorithm}
 \section{Convergence Proof}
 In this section, a convergence analysis of the Algorithm \ref{alg:1} is presented. We construct our convergence proof based on the following assumptions.
 \begin{assumption}\label{ass2}
     Consider the function $f(\cdot) = \sum_{i=1}^{L} f_i(\cdot)$ and let $\text{CP}_f \subset \mathbb{R}^{n}$ represent the set of its critical points. Then, the complement of $\text{CP}_f$ in $\mathbb{R}^{n}$, expressed as $\mathbb{R}^{n} \setminus \text{CP}_f$, is a dense set in $\mathbb{R}^{n}$.
 \end{assumption}
 \begin{assumption}\label{ass3}
 The weight parameter $\alpha^{(k)}$ and the step-size $\beta^{(k)}$ are derived from $\alpha^{(k)}=\frac{a}{k^{\tau_{\alpha}}}$ and $\beta^{(k)}=\frac{b}{k^{\tau_{\beta}}}$ respectively. Here, $0<\tau_{\alpha}<\tau_{\beta}$, $\frac{1}{2}<\tau_{\beta}\leq 1$, and both $a$ and $b$ are positive constants.
\end{assumption}
  While Assumption \ref{ass2} might appear to be technical, it is in fact a relatively mild condition that guides the positioning and quantity of sensors to ensure convergence to the critical points in our algorithm. As an illustration, this assumption is invariably satisfied when employing an odd number of sensors.  Having a non-critical dense set ensures that, through communication and updates, the sensors have enough room to maneuver and collaborate towards converging to a consensus that aligns with the global objective. On the other hand, Assumption \ref{ass3} guarantees that information can be effectively disseminated across the network.

In order to show the regularity of the objective function we provide the following Lemmas.
\begin{lemma}\label{lem1}
    each $f_i(\cdot)$ is locally Lipschitz continuous. 
\end{lemma}
\begin{proof}
    Given that the sub-gradient of each $f_i(\cdot)$ is bounded, it follows that these functions are locally Lipschitz continuous.
\end{proof}
\begin{lemma}\label{lem2}
    There exists a radius $R > 0$ and constants $C_1, C_2 > 0$ such that:
    \begin{equation*}
        \left\langle \frac{\mathbf{g}}{\|\mathbf{g}\|}, \frac{\mathbf{x}}{\|\mathbf{x}\|}\right\rangle \geq C_1 \quad \text{and} \quad \|\mathbf{g}\|\leq C_2\|\mathbf{x}\|,
    \end{equation*} 
 for all $ \mathbf{g} \in \partial f_n(\mathbf{x})$ and $\|\mathbf{x}\| \geq R$. 
\end{lemma}
\begin{proof}
    We select a radius $R = (3 \max_i \max_j | {a}_{ij} | + \max_i d_i )+ 1$. Consequently, $\text{sign}(\|\mathbf{x}-\mathbf{a}_i\|-d_i)=1$. We proceed to examine the first condition of the lemma. It is evident that
\begin{align*}
\left\langle \frac{\mathbf{g}}{\|\mathbf{g}\|}, \frac{\mathbf{x}}{\|\mathbf{x}\|}\right\rangle &= \left\langle \frac{\mathbf{x}-\mathbf{a}_i}{\|\mathbf{x}-\mathbf{a}_i\|}, \frac{\mathbf{x}}{\|\mathbf{x}\|}\right\rangle \\ & \geq \frac{\|\mathbf{x}\|}{\|\mathbf{x}-\mathbf{a}_i\|} - \frac{\mathbf{x}^T\mathbf{a}_i}{\|\mathbf{x}-\mathbf{a}_i\|\|\mathbf{x}\|} \\
&\geq 1- \frac{\|\mathbf{a}_i\|}{\|\mathbf{x}-\mathbf{a}_i\|} \geq 1 - \frac{\|\mathbf{a}_i\|}{2\|\mathbf{a}_i\|} = \frac{1}{2}.
\end{align*}
These inequalities are derived from the conditions that $\|\mathbf{x}\|$  always surpasses $\|\mathbf{x}-\mathbf{a}_i\|$, and that $\|\mathbf{x}\|$ always exceeds $3\|\mathbf{a}_i\|$. These conditions are a direct result of our choice of $R$.  This effectively verifies that  $\left\langle \frac{\mathbf{g}}{\|\mathbf{g}\|}, \frac{\mathbf{x}}{\|\mathbf{x}\|}\right\rangle \geq C_1$ for $C_1 = \frac{1}{2}$. For the second condition of the lemma, we are aware that  $\|\mathbf{g}\| = 1$ when $\|\mathbf{x}-\mathbf{a}_i\| > d_i$. This is always true given our selection of $R$. Also, taking into account that  $1 \leq \|\mathbf{x}\|$, the inequality  $\|\mathbf{g}\|\leq C_2\|\mathbf{x}\|$ holds for $C_2 = 1$.
\end{proof}
\begin{theorem}\label{th:1}
Given that the static graph $\mathcal{G}$ is strongly connected, the objective function fulfills Assumption \ref{ass2}, and both the step-size and weight parameter adhere to Assumption \ref{ass3}, the DSRL algorithm is ensured to converge to a critical point.
\end{theorem}
\begin{proof}
The convergence of our proposed DSRL algorithm is substantiated by verifying that all the prerequisites set forth in \cite[Theorem 7]{swenson2022distributed}, inclusive of the regularity conditions of our robust localization function, are met. The fulfillment of these conditions is demonstrated in Lemma \ref{lem1} and Lemma \ref{lem2}.
\end{proof}

\section{Simulation results}
In this section, we assess the performance of the proposed DSRL algorithm via simulations and draw comparisons with the distributed event localization via the Alternating Direction Method of Multipliers (EL-ADMM) algorithm, as outlined in \cite{zhang2019sensor}. The sensor network $\mathcal{G}$ under consideration comprises $L=31$ nodes, which are uniformly and randomly distributed within a $[-3,3]^{3}$ cubic region, and is fully connected. Connections are established between nodes if their distance is less than $1.75$, ensuring a minimum of $2$ and a maximum of $10$ neighbors per node. In addition, the target is assumed to be uniformly and randomly distributed within the same cubic region. For each iteration, the weight parameter $\alpha^{(k)}$ and the step-size $\beta^{(k)}$ are set to $\frac{0.3}{k^{0.55}}$ and $\frac{3.5}{k^{0.5}}$, respectively.
We employ the root-mean-square error (RMSE), denoted by $\sqrt{\frac{\sum_{i=1}^L||\mathbf{\hat{w}}_i-\mathbf{x}||_2^2}{L}}$, as the performance metric. The results are obtained by averaging the outcomes of $1000$ independent trials.
\begin{figure}[t]
    \centering
    \includegraphics[width=0.441\textwidth]{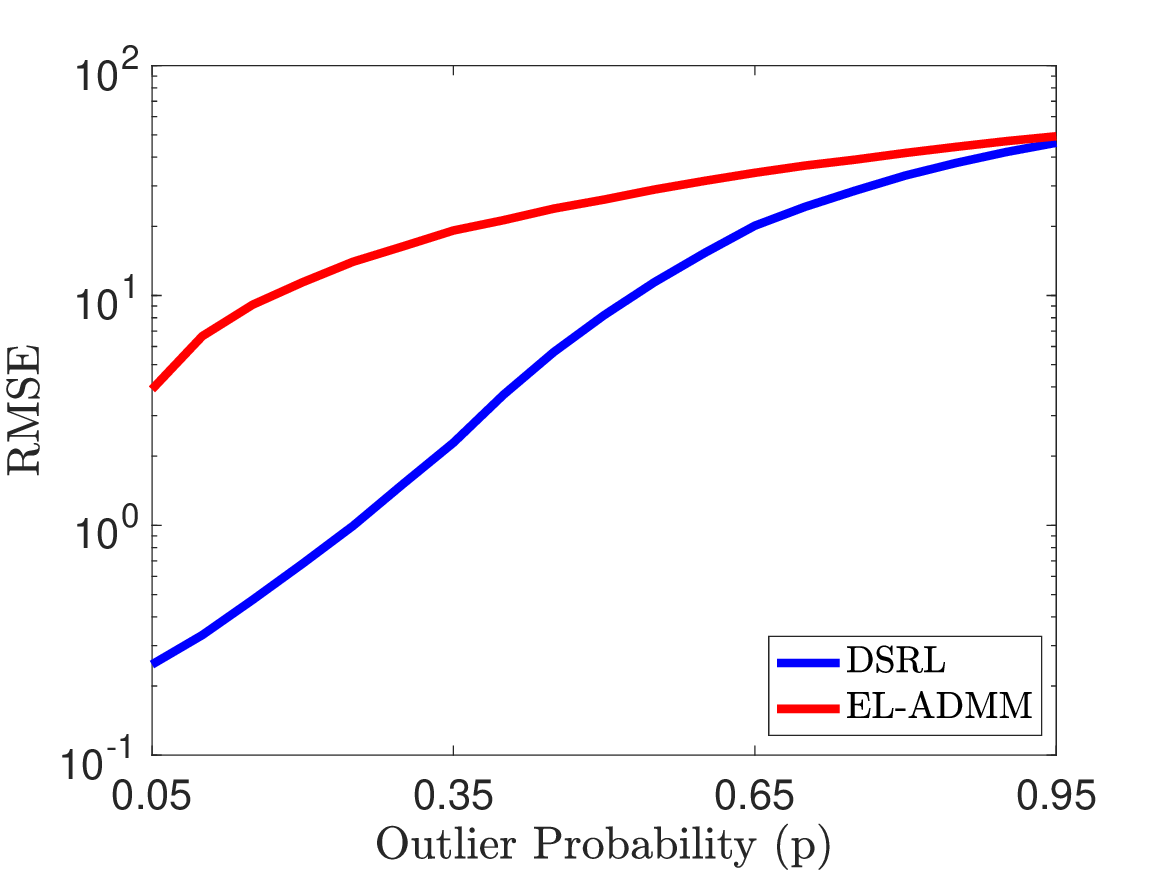}
    \caption{Steady-state RMSE versus the outlier probability.}
    \label{fig1}
\end{figure}

In the first scenario, we focus on comparing the accuracy of the algorithms when handling outliers in the measurements. We introduce outliers at varying probability, denoted as $p$, ranging from $0.05$ to $0.95$ with increments of $0.05$. The outliers, represented as $\epsilon_i$, are drawn from a uniform distribution over the interval $[0, 6\sqrt{3}]$. This is formalized as:
\begin{equation}\label{eq8}
y_i=
\begin{cases}
\epsilon_i, & \text{if } u_i <p,
\\
\|\mathbf{x}-\mathbf{a}_i\|, & \text{otherwise},
\end{cases}
\end{equation}
where $u_i$ is a realization of the uniform distribution on the interval [0,1]. As illustrated in Figure \ref{fig1}, the DSRL algorithm outperforms EL-ADMM consistently across the various frequency of outliers.

\begin{figure}[t]
    \centering
    \includegraphics[width=0.441\textwidth]{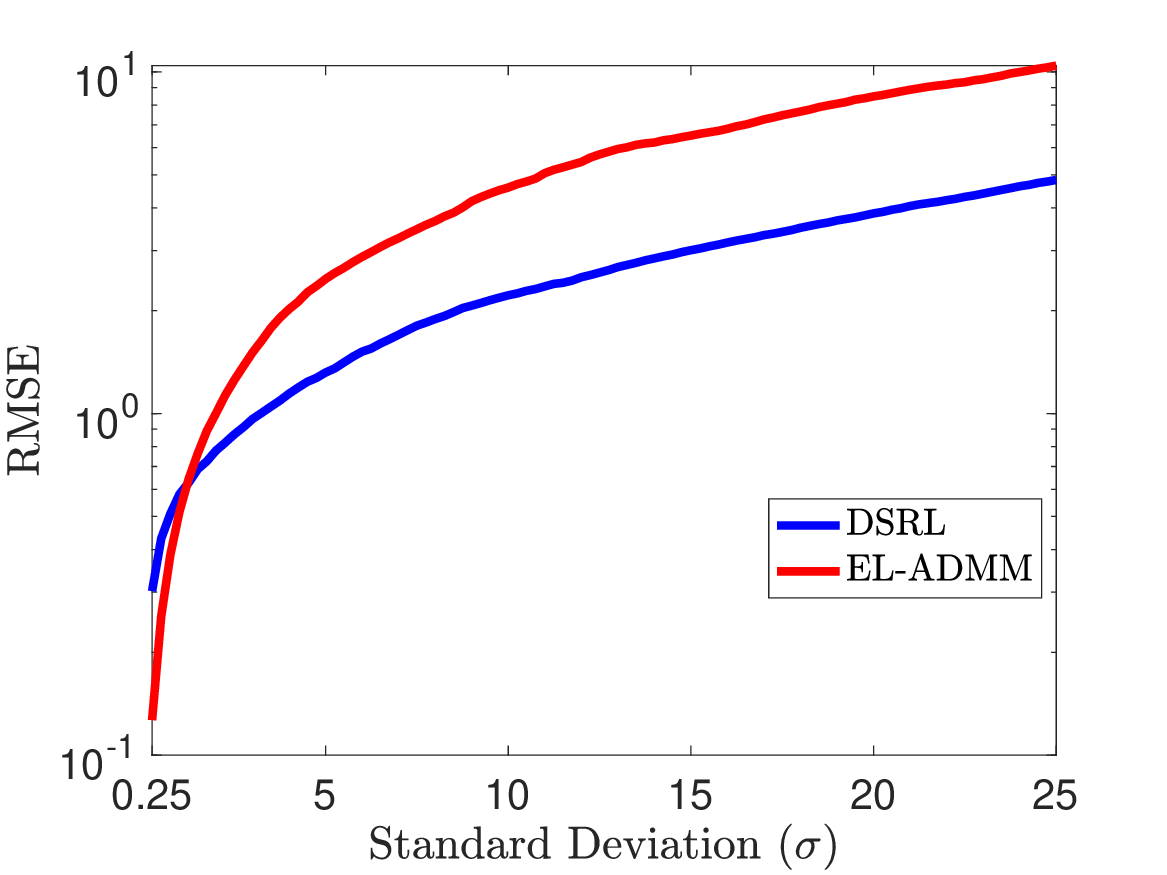}
    \caption{Steady-state RMSE versus the standard deviation of mixture of Laplacian noise.}
    \label{fig2}
\end{figure}

\begin{figure}[t]
    \centering
    \includegraphics[width=0.441\textwidth]{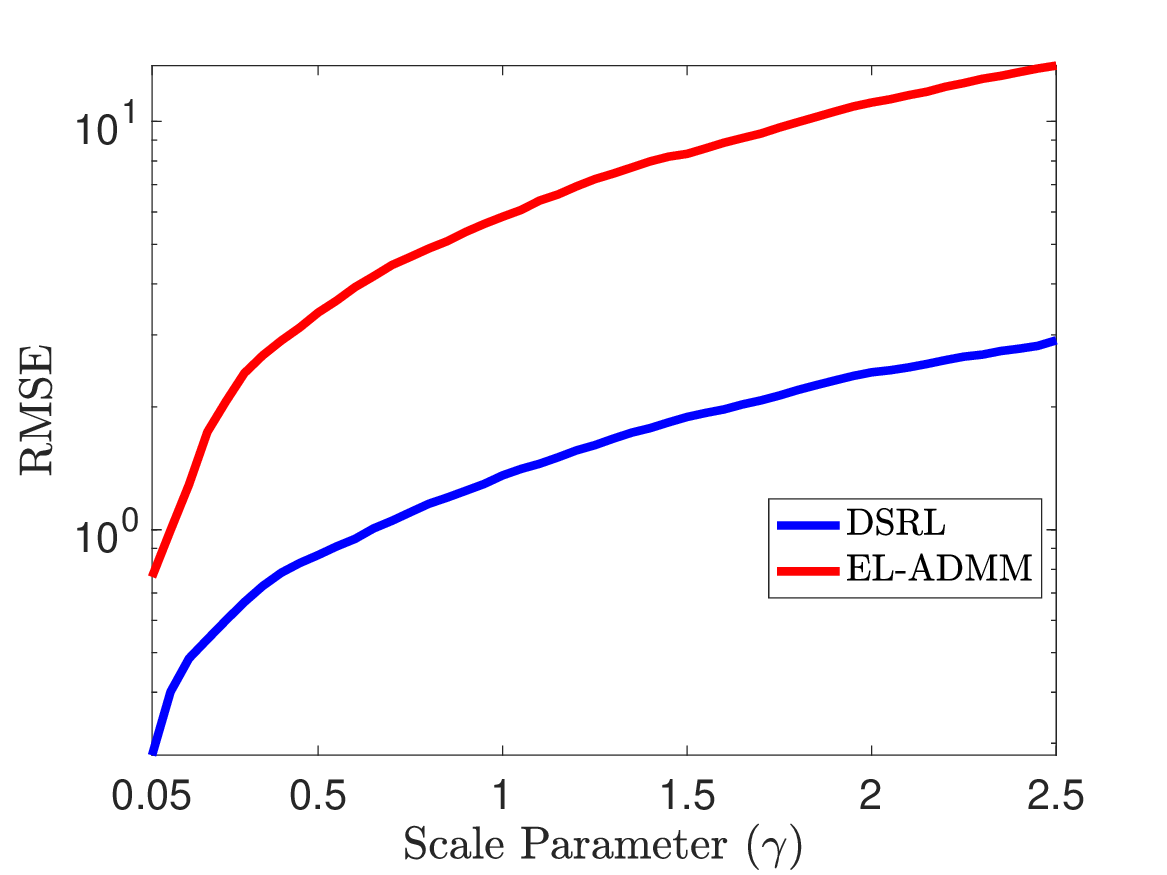}
    \caption{Steady-state RMSE versus scale parameter of Cauchy noise.}
    \label{fig3}
\end{figure}
\begin{figure}[t]
    \centering
    \includegraphics[width=0.441\textwidth]{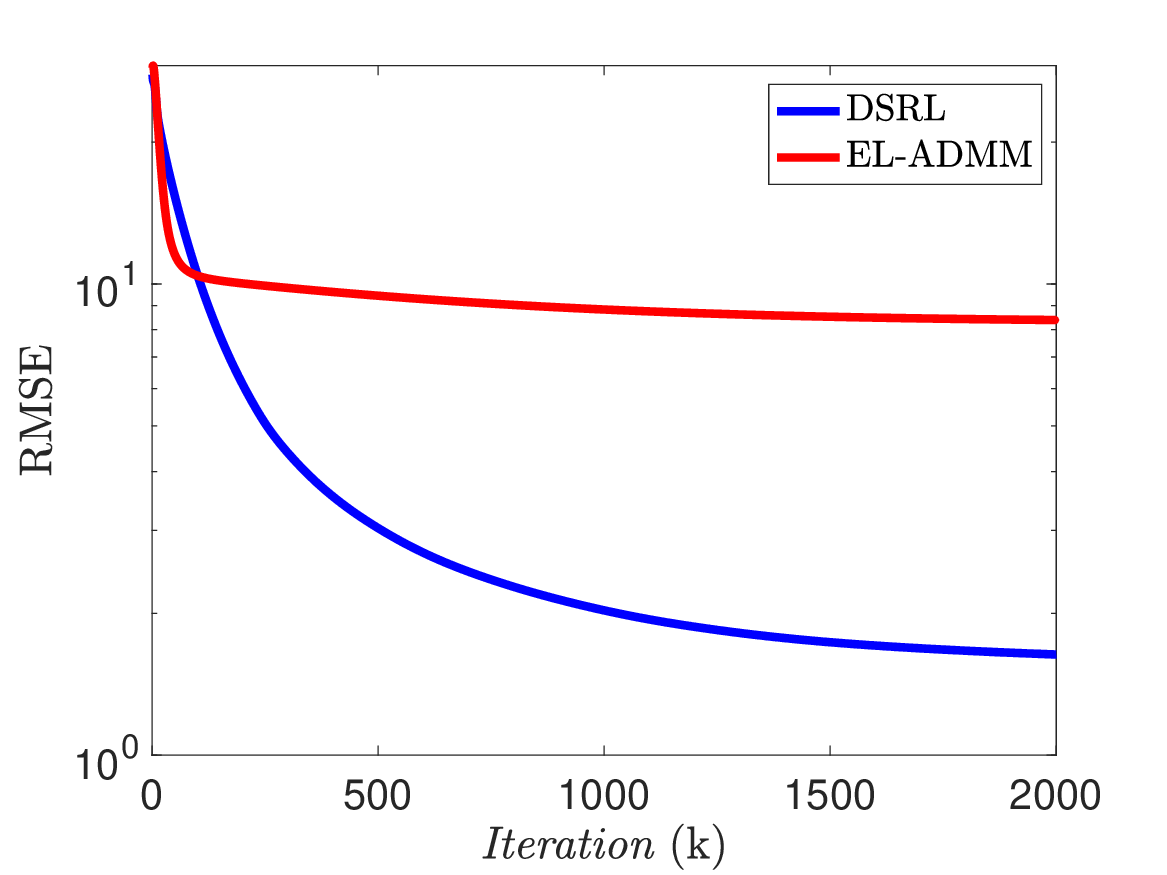}
    \caption{RMSE versus iterations.}
    \label{fig4}
\end{figure}

In the second scenario, we evaluate the performance of the algorithms under noisy conditions, where the noise follows a mixture of Laplacian distributions. The measurements are generated as follows:
\begin{equation}\label{eq9}
y_i = \|\mathbf{x}-\mathbf{a}_i\|+\epsilon_i, \quad \forall i \in {1,\cdots,N},
\end{equation}
where $\{\epsilon_i\}_{i=1}^{N}$ denotes the components of the noise, which are independently and identically sampled from a mixture of Laplacian distributions given by $\sum_{i=1}^{2} c_i \frac{\lambda_i}{2} e^{-\lambda_i|v|}$
Here, $c_1=0.9$, $c_2=0.1$, and $\lambda_2=\frac{\lambda_1}{10}$.  The parameter $\lambda_1$ is related to the standard deviation ($\sigma$) and can be calculated as $\lambda_1 = \sqrt{\frac{\sigma^2}{10.9}}$. The value of $10.9$ in the denominator arises from the specific weights $c_1$ and $c_2$, and it represents the effective contribution of the two components of the Laplacian mixture in relation to the variance. In this setup, $10\%$ of the measurements can be considered outliers. We obtain simulation results for varying standard deviation values, ranging from $0.25$ to $25$, with increments of $0.25$. As illustrated in Figure \ref{fig2}, the DSRL algorithm exhibits superior performance in terms of RMSE, particularly for standard deviation values greater than 1.

In the third scenario, we assess the accuracy of the algorithms in the presence of Cauchy noise. The measurements are formulated as in \eqref{eq9}, where each component $\epsilon_i$ is independently and identically distributed according to the Cauchy distribution with the scale parameter $\gamma$, i.e., $\epsilon_i \sim \text{Cauchy}(0, \gamma)$. We conduct simulations for varying scale parameters, ranging from $0.05$ to $2.5$, in increments of $0.05$. Figure \ref{fig3} illustrates that the DSRL algorithm outperforms EL-ADMM in terms of RMSE across different scale parameters. Additionally, for the scale parameter set at $1$, Figure \ref{fig4} demonstrates that DSRL maintains a competitive convergence speed compared to EL-ADMM.

\section{Conclusion}
This paper presents a sub-gradient algorithm devised for the distributed robust localization problem that directly tackles non-convex and non-smooth objective functions. This proposed algorithm has been rigorously proven to be convergent. Our simulation results highlight the superiority of our algorithm in terms of root-mean-squared error, compared to the existing state-of-the-art EL-ADMM algorithm. Particularly in outlier settings, the algorithm's effectiveness was consistently apparent.

%\newpage

\bibliographystyle{IEEEtran}
\bibliography{strings}

\end{document}